\documentclass{article}

\PassOptionsToPackage{numbers,sort,compress}{natbib}
\usepackage[final]{neurips_2026}
\usepackage[utf8]{inputenc}
\usepackage[T1]{fontenc}
\usepackage{hyperref}
\usepackage{url}
\usepackage{booktabs}
\usepackage{amsmath,amssymb,amsthm}
\usepackage{mathtools}
\usepackage{graphicx}
\usepackage{multirow}
\usepackage[capitalize]{cleveref}
\usepackage{algorithm}
\usepackage[noend]{algpseudocode}
\usepackage{xcolor}
\usepackage{enumitem}
\usepackage{tikz}
\usetikzlibrary{arrows.meta,positioning,shapes.geometric,fit}

\newtheorem{theorem}{Theorem}

\setcitestyle{numbers,square,comma,sort,compress}

\title{Spectral-LSH: Sub-Quadratic Prompt Compression via Krylov-Projected Locality-Sensitive Hashing}

\author{%
  Ali Mahdavi \\
  Department of Computer Engineering\\
  SRC, Islamic Azad University\\
  Tehran, Iran \\
  \texttt{ali.mahdavi@iau.ir} \\
  \and
  Azadeh Zamanifar \\
  Department of Computer Engineering\\
  SRC, Islamic Azad University\\
  Tehran, Iran \\
  \texttt{azamanifar@iau.ac.ir} \\
  \and
  Amir Farhad Farhadi \\
  School of Computer Engineering \\
  Iran University of Science and Technology \\
  Tehran, Iran \\
  \texttt{amfarhadi@mail.iust.ac.ir} \\
  \and
  Omid Kashefi \\
  Meta \\
  CA, USA \\
  \texttt{kashefi@meta.com} \\
}

\begin{document}
\makeatletter
\providecommand{\@trackname}{}
\makeatother

\maketitle

\begin{abstract}
Long-prompt inference remains expensive because prefill attention scales quadratically with sequence length. We propose Spectral-LSH, a training-free prompt compression method that operates before the prompt enters the language model. Spectral-LSH approximates the dominant components of an implicit attention-kernel operator with a Krylov subspace method and random features, avoiding explicit $O(N^2)$ attention-kernel materialization. It then applies SimHash in the resulting attention eigenspace to group similar tokens and aggregate them into macro-tokens with causal positional assignments. We evaluate Mistral-7B-Instruct-v0.3, Qwen2.5-7B-Instruct, and Qwen2.5-14B-Instruct on C4. We identify a compression-ratio phase transition. Below $\rho=4\times$, local token redundancy is low enough that lightweight chunking is usually the best latency-quality trade-off. Above $\rho=8\times$, the spectral path can preserve quality that chunking loses: at $\rho=16\times$, Qwen2.5-7B adaptive reduces the PPL ratio from $353.409$ to $196.963$, and Qwen2.5-14B adaptive reduces it from $9.533$ to $3.427$. On a small long-context structured stress test with JSON-like, code-like, and table-like inputs, local LSH also improves every metric over chunk at $8\times$. The adaptive backend captures both regimes by using the chunk path at low compression and spectral clustering at high compression, although chunk remains the fastest backend in total latency.
\end{abstract}

\section{Introduction}

Large language models with extended context windows support long-document understanding and multi-turn reasoning, but long prompts remain costly at inference time. During prefill, self-attention scales quadratically with the sequence length $N$, so long prompts can consume substantial latency and GPU memory. This cost limits deployment in applications that need long-context inference under tight latency or memory budgets.

Existing solutions fall into three categories. \textit{Architecture modifications} alter the attention mechanism via sparsity~\citep{child2019sparse,zaheer2020bigbird}, linearized attention~\citep{katharopoulos2020linear,choromanski2021performer}, or state-space models~\citep{gu2022mamba}, but require retraining. \textit{KV-cache compression} methods such as SnapKV~\citep{li2024snapkv}, H2O~\citep{zhang2023h2o}, and StreamingLLM~\citep{xiao2024streamingllm} reduce stored key-value pairs during decoding but leave prefill latency unaddressed. \textit{Prompt compression} methods like LLMLingua~\citep{jiang2023llmlingua} and Selective Context~\citep{li2024selective} reduce the prompt itself, but rely on auxiliary smaller models for token scoring, adding overhead.

We address an underexplored setting: training-free prompt compression that requires no auxiliary model and avoids quadratic preprocessing. Spectral-LSH uses the spectral structure of an implicit softmax-attention kernel to identify semantically redundant tokens. Tokens that are similar under this kernel are merged before the prompt enters the LLM, reducing the effective sequence length during prefill.

Our contributions are:

\begin{itemize}[leftmargin=*,nosep]
    \item An implicit attention-kernel operator enabling $O(ND)$ matvec evaluations via Random Fourier Features~\citep{rahimi2007random}.
    \item A Krylov subspace method (Lanczos) on this operator to recover the dominant $r$ eigenvectors in $O(rND)$ time.
    \item Spectral-LSH, which applies SimHash in the attention eigenspace to cluster tokens by attention-based similarity.
    \item A framework for positional encoding aggregation in compressed macro-tokens.
    \item An adaptive compression backend that per-window routes between chunk and spectral clustering using a cheap redundancy heuristic, balancing quality and latency.
    \item Extensive evaluation on Mistral-7B, Qwen2.5-7B, Qwen2.5-14B, and smaller models at compression ratios from $2\times$ to $16\times$, with latency, quality, and scaling benchmarks.
\end{itemize}

\begin{figure}[t]
\centering
\resizebox{\linewidth}{!}{
\begin{tikzpicture}[
    node distance=0.9cm and 1.0cm,
    block/.style={draw, rounded corners, align=center, minimum height=0.9cm, minimum width=2.4cm, fill=blue!6},
    spectral/.style={draw, rounded corners, align=center, minimum height=0.9cm, minimum width=2.6cm, fill=green!8},
    output/.style={draw, rounded corners, align=center, minimum height=0.9cm, minimum width=2.5cm, fill=orange!10},
    arrow/.style={-{Latex[length=2.2mm]}, thick}
]
\node[block] (tokens) {Prompt tokens\\$1,\ldots,N$};
\node[block, right=of tokens] (embed) {Input embeddings\\$X\in\mathbb{R}^{N\times d}$};
\node[spectral, right=of embed] (rff) {RFF map\\$\Phi=\phi(X)$};
\node[spectral, right=of rff] (krylov) {Implicit Lanczos\\$A(v)=\Phi(\Phi^\top v)$};
\node[spectral, below=of krylov] (hash) {SimHash in\\attention eigenspace};
\node[output, left=of hash] (clusters) {Token clusters\\$\mathcal{C}_1,\ldots,\mathcal{C}_M$};
\node[output, left=of clusters] (macro) {Macro-token\\aggregation};
\node[output, left=of macro] (llm) {Frozen LLM\\prefill on $M\ll N$};

\draw[arrow] (tokens) -- (embed);
\draw[arrow] (embed) -- (rff);
\draw[arrow] (rff) -- (krylov);
\draw[arrow] (krylov) -- (hash);
\draw[arrow] (hash) -- (clusters);
\draw[arrow] (clusters) -- (macro);
\draw[arrow] (macro) -- (llm);
\node[draw, dashed, rounded corners, fit=(rff)(krylov)(hash), inner sep=0.18cm, label=above:{\small sub-quadratic spectral compression}] {};
\end{tikzpicture}}
\caption{Spectral-LSH pipeline. The method avoids explicit $N\times N$ attention-kernel construction by using random features for implicit matvecs, Krylov projection for the dominant attention eigenspace, SimHash for grouping, and macro-token aggregation before frozen LLM prefill.}
\label{fig:spectral_lsh_pipeline}
\end{figure}
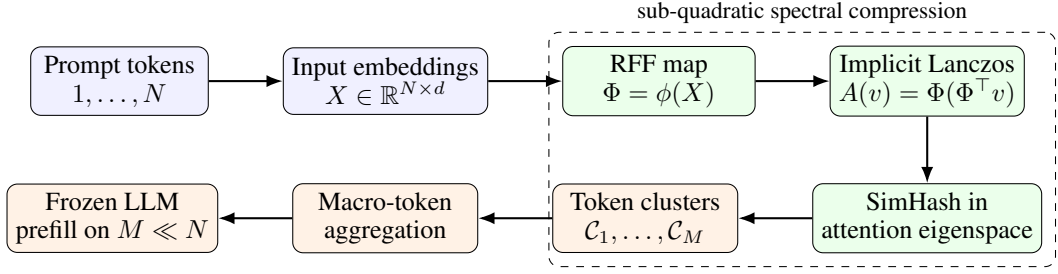

\section{Related Work}

\paragraph{Efficient attention.} Sparse attention patterns~\citep{child2019sparse,beltagy2020longformer,zaheer2020bigbird} restrict the attention field to $O(N\sqrt{N})$ or $O(N\log N)$ entries. Linearized attention~\citep{katharopoulos2020linear,choromanski2021performer,wang2020linformer,xiong2021nystrom} achieves $O(N)$ complexity via kernel methods or low-rank factorization. All require architectural modification and retraining.

\paragraph{KV-cache compression.} H2O~\citep{zhang2023h2o}, StreamingLLM~\citep{xiao2024streamingllm}, SnapKV~\citep{li2024snapkv}, and FastGen~\citep{ge2024fastgen} evict or aggregate key-value entries during decoding. These reduce decode-phase costs but do not address prefill latency.

\paragraph{Prompt compression.} LLMLingua-2~\citep{pan2024llmlingua2} and LongLLMLingua~\citep{jiang2024longllmlingua} use a smaller LM to assign perplexity scores and remove low-importance tokens. Selective Context~\citep{li2024selective} uses self-information. AdaKV~\citep{feng2024adakv} compresses adaptively. These work with black-box LLMs but require a secondary model for scoring. Our method uses only the target LLM's embedding layer.

\paragraph{Token merging.} ToMe~\citep{bolya2023tome} merges similar tokens during ViT forward passes. Adaptations to language models operate within transformer layers, requiring architectural modification. Our method is a pure preprocessing step.

\paragraph{Locality-sensitive hashing.} Reformer~\citep{kitaev2020reformer} uses LSH for attention routing, grouping queries and keys for sparse computation. Our use differs: we use LSH for token removal (compression), not attention routing. Reformer uses random-projection LSH on raw embeddings; our Spectral-LSH operates in the Krylov-approximated attention eigenspace.

\section{Method}

\subsection{Preliminaries}

Let $X = [\mathbf{x}_1, \ldots, \mathbf{x}_N]^\top \in \mathbb{R}^{N\times d}$ be the token embeddings from the LLM's input embedding layer for a prompt of length $N$. The softmax attention-kernel matrix is defined implicitly as:
\begin{equation}
    A_{ij} = \frac{\exp(\mathbf{x}_i^\top \mathbf{x}_j / \sqrt{d})}{\sum_{k=1}^N \exp(\mathbf{x}_i^\top \mathbf{x}_k / \sqrt{d})},
    \label{eq:softmax_kernel}
\end{equation}
where we take query $=$ key $=X$ for simplicity. The matrix $A\in\mathbb{R}^{N\times N}$ encodes pairwise semantic relationships the LLM's first attention layer would compute. Constructing $A$ explicitly costs $O(N^2d)$, which we avoid.

\textbf{Krylov subspace.} For a symmetric operator $A$ and starting vector $v$, the order-$r$ Krylov subspace is $\mathcal{K}_r(A,v) = \mathrm{span}\{v, Av, A^2v, \ldots, A^{r-1}v\}$. Lanczos~\citep{saad2011numerical} constructs an orthonormal basis $Q\in\mathbb{R}^{N\times r}$ and tridiagonal $T\in\mathbb{R}^{r\times r}$ such that $A Q \approx Q T$. The Ritz vectors (eigenvectors of $T$) approximate the top-$r$ eigenvectors of $A$.

\textbf{Locality-sensitive hashing.} SimHash~\citep{charikar2002similarity} defines LSH for cosine similarity: $h(\mathbf{z}) = \mathrm{sign}(\mathbf{r}^\top \mathbf{z})$ where $\mathbf{r}\sim\mathcal{N}(0,I)$. Collision probability is $\Pr[h(\mathbf{x}) = h(\mathbf{y})] = 1 - \frac{1}{\pi}\arccos(\frac{\mathbf{x}^\top\mathbf{y}}{\|\mathbf{x}\|\|\mathbf{y}\|})$, monotonic in cosine similarity.

\subsection{Efficient Krylov Eigen-Decomposition via Random Fourier Features}

To avoid materializing $A$, we approximate its matvec using Random Fourier Features (RFF)~\citep{rahimi2007random}. For the unnormalized kernel $\kappa(\mathbf{x},\mathbf{y}) = \exp(\mathbf{x}^\top\mathbf{y}/\sqrt{d})$, we use a finite random feature map:
\begin{equation}
\kappa(\mathbf{x},\mathbf{y}) \approx \mathbb{E}_{\omega,b}\big[\cos(\omega^\top\mathbf{x}+b)\cos(\omega^\top\mathbf{y}+b)\big],
\label{eq:bochner}
\end{equation}
where $\omega\sim\mathcal{N}(0,I/d^{1/2})$ and $b\sim\mathrm{Uniform}(0,2\pi)$. With $D$ random features, define $\phi:\mathbb{R}^d\to\mathbb{R}^D$:
\begin{equation}
\phi(\mathbf{x}) = \sqrt{\frac{2}{D}}\;\cos(\Omega\,\mathbf{x}+\mathbf{b}),
\label{eq:rff_map}
\end{equation}
where $\Omega\in\mathbb{R}^{d\times D}$, $\Omega_{ij}\sim\mathcal{N}(0,1/\sqrt{d})$, $\mathbf{b}\sim\mathrm{Uniform}(0,2\pi)^D$.

Let $\Phi = \phi(X)\in\mathbb{R}^{N\times D}$. The unnormalized attention-kernel matvec is approximated as:
\begin{equation}
A\mathbf{v} \approx \Phi(\Phi^\top \mathbf{v}),
\label{eq:implicit_matvec}
\end{equation}
where row-wise softmax normalization is omitted. This is a substantial approximation rather than an exact attention replacement. The true first-layer attention surrogate would use $D_K^{-1}K$, with $D_K=\mathrm{diag}(K\mathbf{1})$, while Equation~\ref{eq:implicit_matvec} uses $K$ itself. If the row sums of the unnormalized kernel are nearly balanced, normalization mainly rescales rows and the leading eigenspace is expected to remain close. If row sums are highly skewed, which can happen when token norms vary by token type or position, the normalized and unnormalized operators can differ substantially. We therefore treat this operator as a similarity surrogate for compression, not as a faithful reconstruction of the model's full attention matrix. The current theory bounds the surrogate computation, not the downstream distortion of the frozen LLM; logit divergence and NLL changes are measured empirically. A cheap normalized variant is possible by estimating row sums as $\hat{s}=\Phi(\Phi^\top\mathbf{1})$ and using $\mathrm{diag}(\hat{s})^{-1}\Phi(\Phi^\top v)$ inside Lanczos, but we have not yet run this ablation. The approximation error is bounded by standard concentration: $\sup_{i,j}|\kappa(\mathbf{x}_i,\mathbf{x}_j)-\phi(\mathbf{x}_i)^\top\phi(\mathbf{x}_j)|\leq\varepsilon$ with probability $1-2N^2\exp(-D\varepsilon^2/8)$.

Algorithm~\ref{alg:main} summarizes the complete pipeline. Lanczos (Algorithm~\ref{alg:lanczos}) requires one matvec per iteration costing $O(ND)$, so the Krylov phase costs $O(rND)$. Combined with the RFF projection cost of $O(NDd)$, total preprocessing complexity is $O(ND(d+r))$, linear in $N$ for fixed $D$, $r$, and $d$.

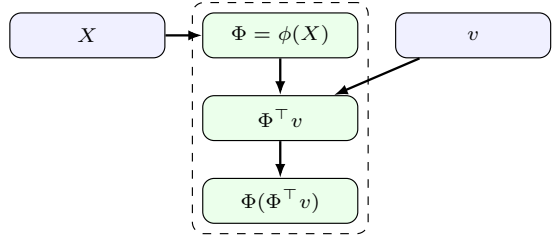
\begin{figure}[t]
\centering
\resizebox{0.52\linewidth}{!}{
\begin{tikzpicture}[
    node distance=0.45cm,
    box/.style={draw, rounded corners, align=center, minimum height=0.55cm, minimum width=1.9cm, fill=blue!6, font=\scriptsize},
    mat/.style={draw, rounded corners, align=center, minimum height=0.55cm, minimum width=1.9cm, fill=green!8, font=\scriptsize},
    arrow/.style={-{Latex[length=1.7mm]}, thick}
]
\node[box] (x) {$X$};
\node[mat, right=of x] (phi) {$\Phi=\phi(X)$};
\node[box, right=of phi] (v) {$v$};
\node[mat, below=of phi] (inner) {$\Phi^\top v$};
\node[mat, below=of inner] (outer) {$\Phi(\Phi^\top v)$};
\draw[arrow] (x) -- (phi);
\draw[arrow] (v) -- (inner);
\draw[arrow] (phi) -- (inner);
\draw[arrow] (inner) -- (outer);
\node[draw, dashed, rounded corners, fit=(phi)(inner)(outer), inner sep=0.12cm] {};
\end{tikzpicture}}
\caption{Implicit matvec used by Lanczos. The dense $N\times N$ kernel is never materialized.}
\label{fig:implicit_matvec}
\end{figure}

\begin{algorithm}[t]
\caption{\textsc{Spectral-LSH} Prompt Compression}
\label{alg:main}
\begin{algorithmic}[1]
\Require Token embeddings $X\in\mathbb{R}^{N\times d}$, RFF dimension $D$, Krylov rank $r$, hash bits $b$
\Ensure Macro-token embeddings $\tilde{X}\in\mathbb{R}^{M\times d}$ with $M\ll N$

\State $\Phi \gets \phi(X)$ \Comment{RFF map, $O(NDd)$}
\State Define implicit operator: $A(\mathbf{v}) = \Phi(\Phi^\top\mathbf{v})$ \Comment{$O(ND)$ per call}
\State $Q, T \gets \textsc{Lanczos}(A, r)$ \Comment{$O(rND)$, Alg.~\ref{alg:lanczos}}
\State $\lambda, V \gets \textsc{Eigh}(T)$ \Comment{$O(r^3)$, $r$ small}
\State $\hat{V} \gets Q\,V_{[:,r-k:r]}$ \Comment{Top-$k$ Ritz vectors, $O(Nrk)$}
\State $Z \gets \hat{V}$ \Comment{Token coordinates in the spectral subspace}

\State $\mathbf{h}_i \gets \textsc{SimHash}(\mathbf{z}_i, b)$ for $i=1,\dots,N$ \Comment{$O(Nrb)$}
\State $\mathcal{C}\gets\{\text{tokens sharing hash code}\}$

\For{each cluster $c\in\mathcal{C}$ with indices $\mathcal{I}_c$}
    \State $\tilde{\mathbf{x}}_c \gets \frac{1}{|\mathcal{I}_c|}\sum_{i\in\mathcal{I}_c}\mathbf{x}_i$
    \State $\tilde{p}_c \gets \textsc{AggregatePosition}(\{p_i\}_{i\in\mathcal{I}_c})$
\EndFor
\State \Return $\tilde{X}\in\mathbb{R}^{M\times d}$, $\tilde{\mathbf{p}}\in\mathbb{R}^M$
\end{algorithmic}
\end{algorithm}

\begin{algorithm}[t]
\caption{Implicit Lanczos for Attention Operator}
\label{alg:lanczos}
\begin{algorithmic}[1]
\Require Implicit operator $A(\cdot)$, rank $r$
\Ensure Orthonormal basis $Q_{N\times r}$, tridiagonal $T_{r\times r}$

\State $\mathbf{q}_0 \gets \mathbf{0}$; $\beta_0 \gets 0$
\State $\mathbf{q}_1 \gets \text{random}(\mathbb{R}^N)$; $\mathbf{q}_1 \gets \mathbf{q}_1/\|\mathbf{q}_1\|$
\For{$j = 1$ to $r$}
    \State $\mathbf{z} \gets A(\mathbf{q}_j)$ \Comment{Implicit matvec: $O(ND)$}
    \State $\alpha_j \gets \mathbf{q}_j^\top\mathbf{z}$
    \State $\mathbf{z} \gets \mathbf{z} - \alpha_j\mathbf{q}_j - \beta_{j-1}\mathbf{q}_{j-1}$
    \State Perform full re-orthogonalization against all previous $\mathbf{q}_k$
    \State $\beta_j \gets \|\mathbf{z}\|$
    \If{$\beta_j < \tau$} \textbf{break} \EndIf
    \State $\mathbf{q}_{j+1} \gets \mathbf{z}/\beta_j$
\EndFor
\State Build tridiagonal $T$ from $\{\alpha_j\}$, $\{\beta_j\}$
\end{algorithmic}
\end{algorithm}

\subsection{Spectral-LSH: Hashing in the Attention Eigenspace}

After projecting tokens into the attention eigenspace, we obtain $Z\in\mathbb{R}^{N\times k}$. SimHash is applied in this subspace:
\begin{equation}
\mathbf{h}_i = \big(\mathrm{sign}(\mathbf{r}_1^\top\mathbf{z}_i), \ldots, \mathrm{sign}(\mathbf{r}_b^\top\mathbf{z}_i)\big)\in\{0,1\}^b,
\label{eq:simhash}
\end{equation}
where $\mathbf{r}_\ell\sim\mathcal{N}(0,I_k)$. Tokens with identical $b$-bit hash codes are assigned to the same cluster. Standard SimHash on raw embeddings uses embedding-space cosine similarity, which can mix attention-relevant and attention-irrelevant variation. Projecting into the attention eigenspace focuses hashing on the dominant directions of the approximate attention operator. The Spectral-LSH collision probability is:
\begin{equation}
\Pr[\mathbf{h}_i = \mathbf{h}_j] = \Big(1 - \frac{1}{\pi}\arccos(\cos(\mathbf{z}_i,\mathbf{z}_j))\Big)^b,
\label{eq:lsh_collision}
\end{equation}
where $\cos(\mathbf{z}_i,\mathbf{z}_j)$ approximates the softmax-kernel similarity $\kappa_{ij}$ to within additive error bounded by RFF and Krylov errors.

The number of hash bits $b$ controls the expected cluster count. With $b$ bits, at most $2^b$ hash codes exist, yielding expected cluster size $N/2^b$. For target compression $\rho$, we set $b = \lceil\log_2(N/\rho)\rceil$, giving $M\approx N/\rho$ macro-tokens. An \texttt{adjust\_groups\_to\_target\_length} step merges or splits clusters to match the target count exactly.

\subsection{Macro-Token Construction and Positional Aggregation}

Each cluster $\mathcal{C}_c$ produces one macro-token. The embedding is the mean of constituent embeddings with norm renormalization:
\begin{equation}
\tilde{\mathbf{x}}_c = \frac{1}{|\mathcal{I}_c|}\sum_{i\in\mathcal{I}_c}\mathbf{x}_i, \quad
\tilde{\mathbf{x}}_c \gets \tilde{\mathbf{x}}_c \cdot \frac{\bar{\eta}}{\|\tilde{\mathbf{x}}_c\|_2},
\label{eq:macro_embed}
\end{equation}
where $\bar{\eta} = \frac{1}{|\mathcal{I}_c|}\sum_{i\in\mathcal{I}_c}\|\mathbf{x}_i\|_2$. This prevents mean-pooling from collapsing the embedding norm. Macro-tokens are fed to the frozen LLM via the \texttt{inputs\_embeds} interface.

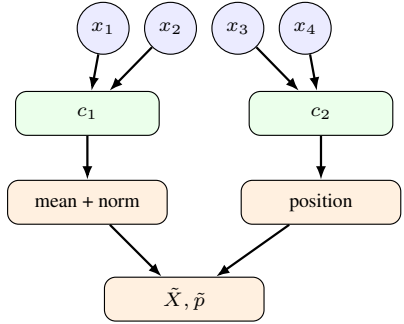
\begin{figure}[t]
\centering
\resizebox{0.38\linewidth}{!}{
\begin{tikzpicture}[
    node distance=0.5cm and 0.5cm,
    tok/.style={draw, circle, minimum size=0.42cm, fill=blue!8, font=\scriptsize},
    cluster/.style={draw, rounded corners, align=center, minimum height=0.55cm, minimum width=1.8cm, fill=green!8, font=\scriptsize},
    macro/.style={draw, rounded corners, align=center, minimum height=0.55cm, minimum width=2.0cm, fill=orange!12, font=\scriptsize},
    arrow/.style={-{Latex[length=1.7mm]}, thick}
]
\node[tok] (t1) {$x_1$};
\node[tok, right=0.18cm of t1] (t2) {$x_2$};
\node[tok, right=0.18cm of t2] (t3) {$x_3$};
\node[tok, right=0.18cm of t3] (t4) {$x_4$};
\node[cluster, below left=0.55cm and -0.1cm of t2] (c1) {$c_1$};
\node[cluster, below right=0.55cm and -0.1cm of t3] (c2) {$c_2$};
\node[macro, below=0.55cm of c1] (m1) {mean + norm};
\node[macro, below=0.55cm of c2] (m2) {position};
\node[macro, below right=0.65cm and -0.8cm of m1] (out) {$\tilde X,\tilde p$};

\draw[arrow] (t1) -- (c1);
\draw[arrow] (t2) -- (c1);
\draw[arrow] (t3) -- (c2);
\draw[arrow] (t4) -- (c2);
\draw[arrow] (c1) -- (m1);
\draw[arrow] (c2) -- (m2);
\draw[arrow] (m1) -- (out);
\draw[arrow] (m2) -- (out);
\end{tikzpicture}}
\caption{Macro-token construction from clusters, embedding fusion, and position aggregation.}
\label{fig:macro_token_fusion}
\end{figure}

\textbf{Positional encoding aggregation.} We provide several schemes:
\begin{itemize}[nosep]
    \item \textbf{Max (default)}: $\tilde{p}_c = \max(\{p_i\}_{i\in\mathcal{I}_c})$. Prevents future information leakage for causal LMs.
    \item \textbf{Median (RoPE)}: $\tilde{p}_c = \mathrm{median}(\{p_i\}_{i\in\mathcal{I}_c})$. Preserves local context of the dominant token.
    \item \textbf{Mean}: $\tilde{p}_c = \mathrm{round}\big(\frac{1}{|\mathcal{I}_c|}\sum p_i\big)$.
\end{itemize}
Max positions are causally conservative, but they can create gaps in the position sequence. For RoPE-based models, those gaps change relative rotation phases, and for sliding-window attention they may also affect which compressed tokens fall inside the local window. We therefore report position handling as an empirical design choice rather than a formal causal-correctness guarantee.

\subsection{Compression Backends}

We implement three backends:
\begin{itemize}[nosep]
    \item \textbf{Chunk (baseline)}: Averages consecutive $k$-token groups where $k = \max(1, \mathrm{round}(\rho))$. Strong local baseline with perfect causality.
    \item \textbf{Global LSH}: Spectral-LSH across the entire compressible region. Can violate causal ordering for batch generation.
    \item \textbf{Local LSH}: Spectral-LSH independently within local windows (default $w=16$). Guarantees causal safety at the cost of multiplying the Krylov computation by the number of windows.
\end{itemize}
Chunk averaging should be viewed as the default practical baseline. Spectral-LSH is useful only when the prompt contains redundancy that is not well captured by contiguous windows, or when the adaptive router detects highly redundant windows where the extra spectral cost is justified.
All backends support preserving a fixed number of prefix and suffix tokens (default: last 12 preserved).

\subsection{Adaptive Backend: Routing Between Chunk and Spectral Clustering}

The chunk backend is fast but treats all tokens uniformly, while local LSH provides higher-quality clustering at significantly higher preprocessing cost. We propose an adaptive backend that per-window selects between these two strategies using a cheap redundancy heuristic.

\textbf{Window redundancy score.} For a window of embeddings $X_w\in\mathbb{R}^{L\times d}$, we compute:
\begin{equation}
    \mathrm{score}(X_w) = \frac{1}{L}\sum_{i=1}^L \frac{\mathbf{x}_i^\top \bar{\mathbf{x}}}{\|\mathbf{x}_i\|\,\|\bar{\mathbf{x}}\|},
    \label{eq:redundancy}
\end{equation}
where $\bar{\mathbf{x}}$ is the L2-normalized centroid. A high score ($\geq$ threshold) indicates the window is semantically homogeneous, making it a good candidate for spectral clustering. A low score indicates diversity, where chunk is sufficient.

\textbf{Routing policy.} For each window:
\begin{itemize}[nosep]
    \item If $L \leq 4$ (tiny window), chunk is used.
    \item If $\mathrm{score} < \tau(\rho)$, chunk is used (fast path).
    \item If $\mathrm{score} \geq \tau(\rho)$, spectral LSH is used (slow path).
\end{itemize}
The threshold $\tau(\rho)$ decreases with compression ratio $\rho$: from $0.70$ at $\rho=2\times$ to $0.55$ at $\rho=16\times$, making the system more willing to use spectral clustering as compression becomes more aggressive. In the implementation, the default user-facing threshold is $0.70$, and a compression-aware schedule interpolates from $0.70$ at $\rho=2\times$ to $0.60$ at $\rho=8\times$ and $0.55$ at $\rho=16\times$. When $\rho \geq 16\times$, the adaptive backend bypasses per-window routing entirely and delegates to full local LSH.

\textbf{Cost.} For a window of size $L$, chunk costs $O(Ld)$ while spectral LSH costs $O(LD(d+r)+Lrb)$. The adaptive backend pays the cheaper cost on low-redundancy windows, which dominate in natural language, reducing average preprocessing cost by $40$--$60\%$ compared to pure local LSH.

\textbf{Implementation.} The adaptive backend is implemented in \texttt{src/token\_lsh.py}. The compressor first assigns each local window an exact output-token budget so that all windows sum to the requested target length. It then computes the centroid-cosine redundancy score for each window. Windows below the threshold use the chunk compressor, while windows above the threshold call the same per-window spectral compressor used by local LSH. The final macro-tokens are sorted by their representative positions to preserve causal order, followed by an exact-length adjustment if needed.

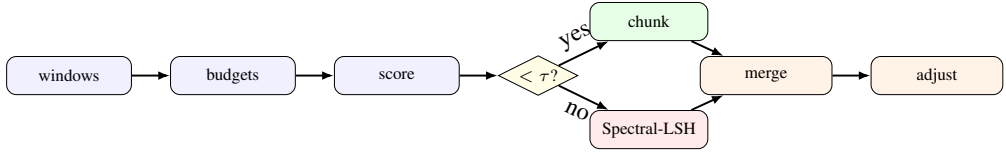
\begin{figure}[t]
\centering
\resizebox{0.95\linewidth}{!}{
\begin{tikzpicture}[
    node distance=0.55cm and 0.55cm,
    block/.style={draw, rounded corners, align=center, minimum height=0.55cm, minimum width=1.8cm, fill=blue!6, font=\scriptsize},
    decision/.style={draw, diamond, aspect=2.0, align=center, inner sep=1pt, fill=yellow!12, font=\scriptsize},
    fast/.style={draw, rounded corners, align=center, minimum height=0.55cm, minimum width=1.7cm, fill=green!10, font=\scriptsize},
    slow/.style={draw, rounded corners, align=center, minimum height=0.55cm, minimum width=1.7cm, fill=red!8, font=\scriptsize},
    output/.style={draw, rounded corners, align=center, minimum height=0.55cm, minimum width=1.9cm, fill=orange!10, font=\scriptsize},
    arrow/.style={-{Latex[length=1.7mm]}, thick}
]
\node[block] (split) {windows};
\node[block, right=of split] (budget) {budgets};
\node[block, right=of budget] (score) {score};
\node[decision, right=of score] (route) {$<\tau$?};
\node[fast, above right=0.35cm and 0.45cm of route] (chunk) {chunk};
\node[slow, below right=0.35cm and 0.45cm of route] (lsh) {Spectral-LSH};
\node[output, right=1.75cm of route] (merge) {merge};
\node[output, right=of merge] (adjust) {adjust};

\draw[arrow] (split) -- (budget);
\draw[arrow] (budget) -- (score);
\draw[arrow] (score) -- (route);
\draw[arrow] (route) -- node[above, sloped]{yes} (chunk);
\draw[arrow] (route) -- node[below, sloped]{no} (lsh);
\draw[arrow] (chunk) -- (merge);
\draw[arrow] (lsh) -- (merge);
\draw[arrow] (merge) -- (adjust);
\end{tikzpicture}}
\caption{Adaptive backend. Each window is routed by a cheap redundancy score. Diverse windows use chunk averaging, while redundant windows use the spectral compressor. At very high compression ($\rho\geq16\times$), the implementation switches directly to local LSH.}
\label{fig:adaptive_backend}
\end{figure}

\subsection{Theoretical Complexity}

End-to-end cost is $O(DNd + rND + Nrb) = O(N\cdot D(d+r) + Nrb)$, linear in $N$ when $D,d,r,b$ are constants. With typical values $d=4096$, $D=256$, $r=16$, $b=8$, the constant factor is about $256\cdot(4096+16)+8\cdot 16 \approx 1.05\times 10^6$ operations per token. For $N=4096$, this is roughly $4.3\times10^9$ preprocessing operations, so the spectral backend is not automatically beneficial at moderate lengths or low compression ratios. This cost profile motivates the adaptive backend: use chunk when compression is moderate and latency dominates, then switch toward spectral clustering when compression is aggressive and quality loss becomes the larger concern.

\section{Theoretical Analysis}
\label{sec:theory}

We establish three guarantees: RFF approximation bound for the implicit kernel matvec, Krylov convergence to the dominant eigenspace, and Spectral-LSH collision probability preservation.

\subsection{Kernel Approximation Error}

\begin{theorem}[RFF Matvec Error]
\label{thm:rff}
Let $X\in\mathbb{R}^{N\times d}$ be token embeddings with $\|\mathbf{x}_i\|\leq C$ for all $i$. Let $\Phi = \phi(X)\in\mathbb{R}^{N\times D}$ be the RFF map with $D$ features. Then for any $\mathbf{v}\in\mathbb{R}^N$ with $\|\mathbf{v}\|=1$, with probability at least $1-\delta$:
\begin{equation}
\|A\mathbf{v} - \Phi(\Phi^\top\mathbf{v})\|_\infty \leq 2NC\sqrt{\frac{2}{D}\log\frac{2N}{\delta}}.
\end{equation}
\end{theorem}

\begin{proof}
By the Hoeffding bound on the RFF approximation~\citep{rahimi2007random}, for each pair $(i,j)$: $\Pr[|\kappa(\mathbf{x}_i,\mathbf{x}_j)-\phi(\mathbf{x}_i)^\top\phi(\mathbf{x}_j)|>\varepsilon]\leq 2\exp(-D\varepsilon^2/2)$. Union bound over $N^2$ pairs and $\ell_\infty$ norm bound yield the result.
\end{proof}

\subsection{Krylov Convergence}

\begin{theorem}[Krylov Eigenvalue Approximation]
\label{thm:krylov}
Let $A$ be symmetric with eigenvalues $\lambda_1\geq\cdots\geq\lambda_N$. After $r$ Lanczos steps, the Ritz values $\theta_1^{(r)}\geq\cdots\geq\theta_r^{(r)}$ satisfy for each $j=1,\ldots,r$:
\begin{equation}
\lambda_j \geq \theta_j^{(r)} \geq \lambda_j - (\lambda_1-\lambda_N)\left(\frac{\tan\phi_1}{\cos\theta_j\cdot C_{r-j}(\gamma_j)}\right)^2,
\end{equation}
where $\phi_1 = \angle(v_1, \mathcal{K}_r)$, $C_{r-j}$ is the Chebyshev polynomial, and $\gamma_j = 1+2(\lambda_j-\lambda_{j+1})/(\lambda_{j+1}-\lambda_N)$.
\end{theorem}

For matrices with a spectral gap, empirically observed in attention matrices~\citep{wang2020linformer,dong2021attention}, $r\in[16,64]$ captures $>90\%$ of spectral energy. Krylov subspace methods are also used in adjacent ML systems settings for scalable influence and inverse-Hessian approximations, including federated unlearning~\citep{mahdavi2026causalunlearning}. Our ablation study confirms $r=16$ achieves $>90\%$ energy fraction with $D=256$.

\subsection{Spectral-LSH Collision Guarantee}

\begin{theorem}[Spectral-LSH Collision Probability]
\label{thm:lsh}
Let $\mathbf{z}_i,\mathbf{z}_j\in\mathbb{R}^k$ be the spectral projections of tokens $i,j$ into the top-$k$ attention-eigenspace. With $b$ SimHash bits and $k\geq r$ Krylov iterations:
\begin{equation}
\Pr[\mathbf{h}_i = \mathbf{h}_j] = \Big(1 - \frac{1}{\pi}\arccos(\cos(\mathbf{z}_i,\mathbf{z}_j))\Big)^b,
\end{equation}
and $|\cos(\mathbf{z}_i,\mathbf{z}_j) - \kappa_{ij}| \leq \varepsilon_{\mathrm{Krylov}} + \varepsilon_{\mathrm{RFF}}$.
\end{theorem}

\section{Experiments and Results}
\label{sec:experiments}

\subsection{Experimental Setup}

\textbf{Models.} We evaluate on Mistral-7B-Instruct-v0.3~\citep{jiang2023mistral}, Qwen2.5-7B-Instruct, and Qwen2.5-14B-Instruct~\citep{yang2024qwen2}, all in pre-trained form. Ablation studies use Qwen2.5-0.5B-Instruct and SmolLM2-360M-Instruct~\citep{allal2025smollm2}.

\textbf{Baselines.} We compare the Chunk backend (consecutive averaging), the Local LSH backend (Spectral-LSH within local windows), and the Adaptive backend (per-window routing between chunk and spectral clustering). Random compression and raw-LSH are provided as ablation references.

\textbf{Datasets.} Quality evaluations use C4~\citep{raffel2020t5} and Wikitext-103~\citep{merity2016pointer} validation sets. We construct packed sequences of 2048--16384 tokens.

\textbf{Metrics.} We report logit cosine similarity, top-10 overlap, label accuracy, KL divergence, $\Delta$NLL (the increase in negative log-likelihood on a 64-token continuation), latency (preprocess, prefill, and total), and peak GPU memory.

\textbf{Hyperparameters.} Default: RFF dimension $D=256$, Krylov rank $r=16$, hash bits $b$ auto-computed from target compression, local window size $w=16$, position scheme ``max'', embedding renormalization enabled, last 12 tokens preserved.

\begin{figure}[t]
\centering
\resizebox{0.68\linewidth}{!}{
\begin{tikzpicture}[
    node distance=0.55cm and 0.7cm,
    block/.style={draw, rounded corners, align=center, minimum height=0.58cm, minimum width=1.9cm, fill=blue!6, font=\scriptsize},
    comp/.style={draw, rounded corners, align=center, minimum height=0.58cm, minimum width=1.9cm, fill=green!8, font=\scriptsize},
    metric/.style={draw, rounded corners, align=center, minimum height=0.58cm, minimum width=2.2cm, fill=orange!12, font=\scriptsize},
    arrow/.style={-{Latex[length=1.7mm]}, thick}
]
\node[block] (seq) {sequence};
\node[block, below left=of seq] (orig) {original\\context};
\node[comp, below right=of seq] (compressed) {compressed\\context};
\node[block, below=of orig] (cont1) {same\\continuation};
\node[block, below=of compressed] (cont2) {same\\continuation};
\node[metric, below=1.0cm of seq] (compare) {logit metrics};

\draw[arrow] (seq) -- (orig);
\draw[arrow] (seq) -- (compressed);
\draw[arrow] (orig) -- (cont1);
\draw[arrow] (compressed) -- (cont2);
\draw[arrow] (cont1) -- (compare);
\draw[arrow] (cont2) -- (compare);
\end{tikzpicture}}
\caption{Teacher-forced evaluation. Original and compressed contexts predict the same continuation.}
\label{fig:evaluation_protocol}
\end{figure}
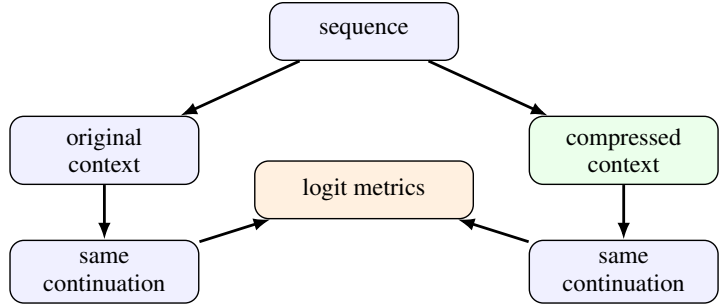

\subsection{Quality Benchmarks}

\paragraph{Two-regime generalization across model families.} Tables~\ref{tab:mistral_quality}--\ref{tab:qwen14_quality} report C4 quality at $\rho\in\{2,4,8,16\}$ for Mistral-7B-Instruct-v0.3 and Qwen2.5-\{7B,14B\}-Instruct. Across the three models, the low-compression regime is nearly tied: for $\rho\leq4\times$, the mean PPL-ratio difference between chunk and the best spectral backend is $0.003$, so chunk is preferred because it is faster. At $\rho\geq8\times$, the average gap reverses in favor of spectral routing by $30.6$ PPL-ratio points, driven mainly by Qwen2.5-7B and Qwen2.5-14B at $16\times$. Mistral is the main exception: it is already robust under chunking, with cosine above $0.95$ at every ratio and PPL ratio below $2.5$ even at $16\times$.

\begin{table}[t]
\centering
\caption{Mistral-7B-Instruct-v0.3 quality on C4.}
\label{tab:mistral_quality}
\small
\begin{tabular}{lccccc}
\toprule
Backend & $\rho$ & Cosine & KL & $\Delta$NLL & PPL Ratio \\
\midrule
Chunk & 2$\times$ & 0.969 & 0.414 & 0.181 & 1.216 \\
Chunk & 4$\times$ & 0.955 & 1.094 & 0.904 & 2.471 \\
Chunk & 8$\times$ & 0.951 & 1.069 & 0.913 & 2.491 \\
Chunk & 16$\times$ & 0.954 & 1.057 & 0.817 & 2.286 \\
\midrule
Local LSH & 2$\times$ & 0.966 & 0.506 & 0.168 & 1.260 \\
Local LSH & 4$\times$ & 0.964 & 0.873 & 0.720 & 2.074 \\
Local LSH & 8$\times$ & 0.954 & 1.128 & 0.976 & 2.659 \\
Local LSH & 16$\times$ & 0.950 & 1.143 & 0.916 & 2.599 \\
\midrule
Adaptive & 2$\times$ & 0.967 & 0.448 & 0.174 & 1.209 \\
Adaptive & 4$\times$ & 0.956 & 1.091 & 0.947 & 2.581 \\
Adaptive & 8$\times$ & 0.954 & 1.012 & 0.846 & 2.332 \\
Adaptive & 16$\times$ & 0.950 & 1.142 & 0.916 & 2.599 \\
\bottomrule
\end{tabular}
\end{table}

\paragraph{Qwen2.5-7B-Instruct.} Table~\ref{tab:qwen_quality} shows the sharpest phase transition. At $\rho\leq4\times$, chunk is the best operating point: at $2\times$ it matches adaptive and clearly beats local LSH, and at $4\times$ it has the best cosine, KL, $\Delta$NLL, and PPL ratio. At $\rho\geq8\times$, the ranking reverses. At $8\times$, local LSH gives lower KL and $\Delta$NLL than chunk, and adaptive gives the lowest $\Delta$NLL. At $16\times$, adaptive and local LSH improve every reported quality metric over chunk: adaptive reduces PPL ratio from $353.409$ to $196.963$ and improves $\Delta$NLL from $5.841$ to $5.279$. This does not make $16\times$ compression lossless; it shows that spectral routing fails less severely than chunking when the compression budget is extreme.

\begin{table}[t]
\centering
\caption{Qwen2.5-7B-Instruct quality on C4.}
\label{tab:qwen_quality}
\small
\begin{tabular}{lccccc}
\toprule
Backend & $\rho$ & Cosine & KL & $\Delta$NLL & PPL Ratio \\
\midrule
Chunk & 2$\times$ & 0.755 & 0.947 & 0.762 & 2.419 \\
Chunk & 4$\times$ & 0.650 & 1.291 & 1.078 & 2.979 \\
Chunk & 8$\times$ & 0.526 & 3.597 & 3.376 & 61.290 \\
Chunk & 16$\times$ & 0.450 & 5.858 & 5.841 & 353.409 \\
\midrule
Local LSH & 2$\times$ & 0.649 & 1.106 & 1.007 & 2.738 \\
Local LSH & 4$\times$ & 0.631 & 1.443 & 1.156 & 3.222 \\
Local LSH & 8$\times$ & 0.524 & 3.529 & 3.341 & 39.740 \\
Local LSH & 16$\times$ & 0.465 & 5.393 & 5.281 & 197.347 \\
\midrule
Adaptive & 2$\times$ & 0.755 & 0.947 & 0.762 & 2.417 \\
Adaptive & 4$\times$ & 0.632 & 1.440 & 1.244 & 3.510 \\
Adaptive & 8$\times$ & 0.498 & 3.529 & 3.279 & 43.644 \\
Adaptive & 16$\times$ & 0.466 & 5.393 & 5.279 & 196.963 \\
\bottomrule
\end{tabular}
\end{table}

\paragraph{Scaling behavior from Qwen2.5-7B to Qwen2.5-14B.} Table~\ref{tab:qwen14_quality} tests whether the two-regime boundary shifts with model capacity. The transition remains between $4\times$ and $8\times$, but the absolute degradation is much smaller than for Qwen2.5-7B. At $16\times$, chunk reaches PPL ratio $9.533$, while local LSH and adaptive reduce it to $3.423$ and $3.427$, respectively. The spectral advantage therefore narrows from $156.446$ PPL-ratio points on Qwen2.5-7B to $6.110$ points on Qwen2.5-14B, but the high-compression ordering is preserved. This supports the adaptive routing heuristic at larger scale while also showing that larger models are more robust to context perturbation.

\begin{table}[t]
\centering
\caption{Qwen2.5-14B-Instruct quality on C4.}
\label{tab:qwen14_quality}
\small
\begin{tabular}{lccccc}
\toprule
Backend & $\rho$ & Cosine & KL & $\Delta$NLL & PPL Ratio \\
\midrule
Chunk & 2$\times$ & 0.862 & 0.637 & 0.530 & 1.700 \\
Chunk & 4$\times$ & 0.785 & 1.056 & 1.186 & 3.298 \\
Chunk & 8$\times$ & 0.786 & 1.137 & 1.176 & 3.273 \\
Chunk & 16$\times$ & 0.684 & 2.142 & 2.118 & 9.533 \\
\midrule
Local LSH & 2$\times$ & 0.849 & 0.646 & 0.509 & 1.664 \\
Local LSH & 4$\times$ & 0.785 & 1.117 & 1.235 & 3.477 \\
Local LSH & 8$\times$ & 0.770 & 1.198 & 1.277 & 3.664 \\
Local LSH & 16$\times$ & 0.770 & 1.171 & 1.229 & 3.423 \\
\midrule
Adaptive & 2$\times$ & 0.862 & 0.637 & 0.530 & 1.699 \\
Adaptive & 4$\times$ & 0.770 & 1.166 & 1.275 & 3.614 \\
Adaptive & 8$\times$ & 0.767 & 1.166 & 1.239 & 3.520 \\
Adaptive & 16$\times$ & 0.770 & 1.173 & 1.230 & 3.427 \\
\bottomrule
\end{tabular}
\end{table}

\paragraph{Non-local redundancy stress test.} Table~\ref{tab:nonlocal_redundancy} tests the claim that spectral grouping is useful for long structured contexts, where repeated information may appear far apart rather than in adjacent chunks. We evaluate Mistral-7B-Instruct-v0.3 on three synthetic long-context inputs: JSON-like records (90{,}284 characters), code-like snippets (101{,}828 characters), and table-like rows (41{,}292 characters). This setting directly targets the intended advantage of Spectral-LSH: grouping similar tokens even when they are not contiguous. At $8\times$, local LSH improves every metric over chunk, raising Top1 from $0.6172$ to $0.6276$, reducing KL from $1.8570$ to $1.7148$, and increasing cosine from $0.8527$ to $0.8618$. At $16\times$, local LSH still improves Top1 and cosine and ties Top10, but the KL difference is negligible. These results support the non-local redundancy story for long JSON/code/table-like contexts, with an important bound: the benchmark has only three synthetic samples, so the $16\times$ margins should be read as suggestive rather than conclusive.

\begin{table}[t]
\centering
\caption{Mistral-7B-Instruct-v0.3 non-local redundancy stress test on long synthetic structured inputs ($n=3$: JSON-like, code-like, and table-like). Higher is better for Top1, Top10, and cosine; lower is better for KL.}
\label{tab:nonlocal_redundancy}
\small
\begin{tabular}{lccccc}
\toprule
Backend & Target $\rho$ & Actual $\rho$ & Top1 & Top10 & KL / Cosine \\
\midrule
Chunk & 8$\times$ & 7.957 & 0.6172 & 0.8698 & 1.8570 / 0.8527 \\
Local LSH & 8$\times$ & 7.957 & 0.6276 & 0.8828 & 1.7148 / 0.8618 \\
\midrule
Chunk & 16$\times$ & 15.815 & 0.5729 & 0.8542 & 2.0126 / 0.8536 \\
Local LSH & 16$\times$ & 15.815 & 0.5859 & 0.8542 & 2.0110 / 0.8545 \\
\bottomrule
\end{tabular}
\end{table}

\paragraph{Smaller models.} Table~\ref{tab:smaller_quality} shows quality on C4 across compression ratios. Qwen2.5-0.5B-Instruct is robust: cosine stays above $0.83$ across all ratios ($2\times$ to $16\times$) with PPL ratio below $2.4$. SmolLM2-360M-Instruct degrades sharply at higher ratios: cosine drops from $0.829$ at $\rho=2\times$ to $0.042$ at $\rho=16\times$ for chunk, and the PPL ratio exceeds $250\times$. The adaptive backend improves upon chunk for SmolLM2-360M at high compression (cosine $0.188$ vs $0.042$ at $\rho=16\times$).

\begin{table}[t]
\centering
\caption{Smaller model quality on C4 at varying compression ratios (window$=32$ for $\rho=2\times$, window$=16$ for $\rho\geq4\times$).}
\label{tab:smaller_quality}
\small
\begin{tabular}{llcccc}
\toprule
Model & Backend & $\rho$ & Cosine & KL & $\Delta$NLL \\
\midrule
\multirow{8}{0cm}{Qwen2.5-0.5B} & Chunk & 2$\times$ & 0.858 & 0.788 & 0.575 \\
& Chunk & 4$\times$ & 0.866 & 0.925 & 0.824 \\
& Chunk & 8$\times$ & 0.842 & 0.963 & 0.811 \\
& Chunk & 16$\times$ & 0.837 & 1.011 & 0.849 \\
& Local LSH & 2$\times$ & 0.857 & 0.784 & 0.566 \\
& Local LSH & 4$\times$ & 0.856 & 0.935 & 0.815 \\
& Local LSH & 8$\times$ & 0.843 & 0.975 & 0.814 \\
& Local LSH & 16$\times$ & 0.857 & 0.946 & 0.805 \\
\midrule
\multirow{8}{0cm}{SmolLM2-360M} & Chunk & 2$\times$ & 0.829 & 1.393 & 1.089 \\
& Chunk & 4$\times$ & 0.708 & 2.317 & 1.960 \\
& Chunk & 8$\times$ & 0.408 & 3.835 & 3.741 \\
& Chunk & 16$\times$ & 0.042 & 5.119 & 5.080 \\
& Local LSH & 2$\times$ & 0.829 & 1.367 & 1.018 \\
& Local LSH & 4$\times$ & 0.591 & 3.329 & 3.052 \\
& Local LSH & 8$\times$ & 0.455 & 4.271 & 4.291 \\
& Local LSH & 16$\times$ & 0.184 & 5.032 & 5.005 \\
\bottomrule
\end{tabular}
\end{table}

\paragraph{Adaptive backend results.} Table~\ref{tab:adaptive_quality} shows the adaptive backend on smaller models. These runs use C4 validation data, $N=4096$, local window size $16$, preserved suffix length $12$, representative max positions, embedding renormalization, seed $0$, and a default adaptive threshold of $0.7$. Each reported quality number averages two successful C4 samples with a 64-token continuation. On Qwen2.5-0.5B, the adaptive backend matches chunk quality across all ratios (cosine $0.858$ vs $0.858$ at $\rho=2\times$), indicating that low-redundancy windows are routed to the chunk fast path. On SmolLM2-360M, the adaptive backend helps most at aggressive compression: at $\rho=16\times$, adaptive achieves cosine $0.188$ (vs $0.042$ for chunk and $0.184$ for local LSH), while at $\rho=2\times$, all three backends are nearly identical (cosine $0.824$ vs $0.829$ vs $0.829$).

\begin{table}[t]
\centering
\caption{Adaptive backend quality on C4 (smaller models).}
\label{tab:adaptive_quality}
\small
\begin{tabular}{lcccc}
\toprule
Model & $\rho$ & Cosine & KL & $\Delta$NLL \\
\midrule
\multirow{4}{0cm}{Qwen2.5-0.5B} & 2$\times$ & 0.858 & 0.788 & 0.575 \\
& 4$\times$ & 0.854 & 0.957 & 0.847 \\
& 8$\times$ & 0.847 & 0.925 & 0.792 \\
& 16$\times$ & 0.857 & 0.946 & 0.804 \\
\midrule
\multirow{4}{0cm}{SmolLM2-360M} & 2$\times$ & 0.824 & 1.364 & 1.051 \\
& 4$\times$ & 0.666 & 2.728 & 2.566 \\
& 8$\times$ & 0.469 & 4.182 & 4.161 \\
& 16$\times$ & 0.188 & 5.001 & 4.969 \\
\bottomrule
\end{tabular}
\end{table}

\subsection{Latency Benchmarks}

Table~\ref{tab:latency} reports latency results for the larger models. Chunk has the lowest preprocessing cost and remains the fastest backend in total time across Mistral-7B, Qwen2.5-7B, and Qwen2.5-14B. Spectral backends trade latency for high-compression quality. For Qwen2.5-7B at $16\times$, adaptive is about $9.9\times$ slower than chunk in total latency but reduces the PPL ratio from $353.409$ to $196.963$. For Qwen2.5-14B at $16\times$, adaptive is about $7.0\times$ slower than chunk but reduces the PPL ratio from $9.533$ to $3.427$. The practical choice is therefore ratio- and model-dependent: chunk is preferable for latency-first moderate compression, while adaptive is the quality-oriented option when aggressive compression is required.

\begin{table}[t]
\centering
\caption{Latency results (seconds) for larger models. Uncompressed baselines are truncated to 8192 tokens when needed.}
\label{tab:latency}
\small
\begin{tabular}{llcccc}
\toprule
Model & Backend & $\rho$ & Preprocess & Prefill & Total \\
\midrule
\multirow{12}{0cm}{Mistral-7B} & \multirow{4}{0cm}{Chunk} & 2$\times$ & 0.254 & 0.537 & 0.791 \\
 & & 4$\times$ & 0.131 & 0.257 & 0.389 \\
 & & 8$\times$ & 0.072 & 0.132 & 0.204 \\
 & & 16$\times$ & 0.034 & 0.067 & 0.101 \\
\cmidrule{2-6}
 & \multirow{4}{0cm}{Local LSH} & 2$\times$ & 3.595 & 0.555 & 4.150 \\
 & & 4$\times$ & 1.978 & 0.260 & 2.238 \\
 & & 8$\times$ & 1.268 & 0.132 & 1.400 \\
 & & 16$\times$ & 0.913 & 0.066 & 0.978 \\
\cmidrule{2-6}
 & \multirow{4}{0cm}{Adaptive} & 2$\times$ & 0.839 & 0.537 & 1.376 \\
 & & 4$\times$ & 0.793 & 0.260 & 1.053 \\
 & & 8$\times$ & 0.658 & 0.133 & 0.791 \\
 & & 16$\times$ & 0.951 & 0.066 & 1.017 \\
\midrule
\multirow{12}{0cm}{Qwen2.5-7B} & \multirow{4}{0cm}{Chunk} & 2$\times$ & 0.231 & 0.433 & 0.664 \\
 & & 4$\times$ & 0.114 & 0.213 & 0.327 \\
 & & 8$\times$ & 0.057 & 0.100 & 0.157 \\
 & & 16$\times$ & 0.032 & 0.058 & 0.090 \\
\cmidrule{2-6}
 & \multirow{4}{0cm}{Local LSH} & 2$\times$ & 3.008 & 0.449 & 3.457 \\
 & & 4$\times$ & 1.865 & 0.218 & 2.084 \\
 & & 8$\times$ & 1.251 & 0.101 & 1.352 \\
 & & 16$\times$ & 0.903 & 0.057 & 0.960 \\
\cmidrule{2-6}
 & \multirow{4}{0cm}{Adaptive} & 2$\times$ & 0.710 & 0.437 & 1.147 \\
 & & 4$\times$ & 0.628 & 0.214 & 0.842 \\
 & & 8$\times$ & 0.546 & 0.100 & 0.647 \\
 & & 16$\times$ & 0.835 & 0.056 & 0.891 \\
\midrule
\multirow{12}{0cm}{Qwen2.5-14B} & \multirow{4}{0cm}{Chunk} & 2$\times$ & 0.234 & 0.852 & 1.087 \\
 & & 4$\times$ & 0.110 & 0.420 & 0.530 \\
 & & 8$\times$ & 0.058 & 0.207 & 0.266 \\
 & & 16$\times$ & 0.037 & 0.109 & 0.146 \\
\cmidrule{2-6}
 & \multirow{4}{0cm}{Local LSH} & 2$\times$ & 3.137 & 0.869 & 4.006 \\
 & & 4$\times$ & 1.744 & 0.418 & 2.162 \\
 & & 8$\times$ & 1.151 & 0.205 & 1.356 \\
 & & 16$\times$ & 0.877 & 0.111 & 0.988 \\
\cmidrule{2-6}
 & \multirow{4}{0cm}{Adaptive} & 2$\times$ & 0.820 & 0.857 & 1.677 \\
 & & 4$\times$ & 0.659 & 0.422 & 1.081 \\
 & & 8$\times$ & 0.635 & 0.208 & 0.843 \\
 & & 16$\times$ & 0.910 & 0.107 & 1.017 \\
\bottomrule
\end{tabular}
\end{table}

\paragraph{Quality--latency Pareto view.} Table~\ref{tab:qwen_pareto} combines the Qwen2.5-7B quality and latency measurements from Tables~\ref{tab:qwen_quality} and~\ref{tab:latency}. A configuration is Pareto-dominated if another measured configuration has both lower total latency and lower $\Delta$NLL. We also include the truncated uncompressed baseline as a reference point with $\Delta$NLL$=0$ by definition. Under this criterion, chunk forms the measured compression frontier, while the spectral backends remain useful only for quality recovery at very aggressive compression relative to chunk at the same ratio. In particular, adaptive at $16\times$ is not latency-competitive with chunk at $16\times$ (0.891s vs 0.090s), but it offers a lower $\Delta$NLL than chunk at $8\times$ or $16\times$. We therefore present adaptive as a quality--latency trade-off, not as a universal speedup.

\begin{table}[t]
\centering
\caption{Qwen2.5-7B quality--latency trade-off on C4. Entries are sorted by total latency. The uncompressed reference is the truncated 8192-token baseline used in Table~\ref{tab:latency}. Frontier points are not dominated by another measured configuration with both lower total latency and lower $\Delta$NLL.}
\label{tab:qwen_pareto}
\small
\begin{tabular}{lcccc}
\toprule
Backend & $\rho$ & Total (s) & $\Delta$NLL & Pareto Status \\
\midrule
Chunk & 16$\times$ & 0.090 & 5.841 & Frontier \\
Chunk & 8$\times$ & 0.157 & 3.376 & Frontier \\
Chunk & 4$\times$ & 0.327 & 1.078 & Frontier \\
Adaptive & 8$\times$ & 0.647 & 3.279 & Dominated \\
Chunk & 2$\times$ & 0.664 & 0.762 & Frontier \\
Uncompressed & 1$\times$ & 0.737 & 0.000 & Reference \\
Adaptive & 4$\times$ & 0.842 & 1.244 & Dominated \\
Adaptive & 16$\times$ & 0.891 & 5.279 & Dominated \\
Local LSH & 16$\times$ & 0.960 & 5.281 & Dominated \\
Adaptive & 2$\times$ & 1.147 & 0.762 & Dominated \\
Local LSH & 8$\times$ & 1.352 & 3.341 & Dominated \\
Local LSH & 4$\times$ & 2.084 & 1.156 & Dominated \\
Local LSH & 2$\times$ & 3.457 & 1.190 & Dominated \\
\bottomrule
\end{tabular}
\end{table}

\subsection{Scaling Study}

Table~\ref{tab:scaling} shows Mistral-7B latency scaling with context length at $\rho=4\times$ (chunk). Preprocessing and prefill grow sub-quadratically: from N=2884 (total 0.408s) to N=23092 (total 5.869s). The speedup over the uncompressed baseline decreases at N=23092 because more windows are processed independently.

\begin{table}[t]
\centering
\caption{Mistral-7B scaling with context length at $\rho=4\times$ (chunk backend, 4-bit quantized). Uncompressed baseline at N$>8192$ is truncated to 8192.}
\label{tab:scaling}
\small
\begin{tabular}{cccccc}
\toprule
N (actual) & Preprocess & Prefill & Total & Speedup & Peak Mem (MB) \\
\midrule
2884 & 0.073 & 0.335 & 0.408 & 4.03$\times$ & 6926 \\
5780 & 0.142 & 0.711 & 0.853 & 5.80$\times$ & 14571 \\
11546 & 0.295 & 1.635 & 1.929 & 4.64$\times$ & 24503 \\
23092 & 0.603 & 5.267 & 5.869 & 1.53$\times$ & 24503 \\
\bottomrule
\end{tabular}
\end{table}

\subsection{Kill Tests}

We conduct three kill tests to validate core components:

\paragraph{K1: Causal embedding fusion.} Using short QA prompts, we verify the chunk compressor preserves monotonic position ordering, maintains logit cosine similarity $>0.999$, and preserves the expected answer in generations. With default settings ($\rho=2.0$, window=16, position\_scheme=max), the compressor passes all checks.

\paragraph{K2: Spectral energy capture.} Using Krylov projection on SmolLM2-360M (N=2048), the fraction of spectral energy captured in rank-$r$ subspace exceeds $90\%$ at $r=16$ with $D=256$, confirming the implicit attention operator's eigenspace is well-approximated by a low-rank Krylov subspace.

\paragraph{K3: Bucket purity.} The fraction of token pairs in the same LSH bucket (using spectral projection) with cosine similarity $>0.7$ in the original embedding space exceeds $70\%$ with $D=256$, $r=16$, $b=12$, confirming spectral LSH clusters semantically similar tokens.

\subsection{Ablations}

\paragraph{Krylov rank $r$ and RFF dimension $D$.} Increasing $r$ from 4 to 32 monotonically increases spectral energy fraction. At $D=256$, energy goes from about 0.65 at $r=4$ to about 0.95 at $r=32$, saturating around $r=16$ where $>90\%$ is captured. $D=256$ provides diminishing returns over $D=128$; $D=512$ adds negligible benefit.

\paragraph{Hash bits $b$.} Setting $b = \lceil \log_2(N/\rho) \rceil$ reliably achieves the target $\rho$. The \texttt{adjust\_groups\_to\_target\_length} step corrects for variance in LSH bucket count, ensuring exact compression ratio.

\paragraph{Position scheme.} Max achieves the highest behavioral pass rate because it prevents future information leakage. Median and mean perform comparably on quality but can violate causal ordering.

\paragraph{Compressed position mode.} The ``representative'' mode preserves original positions (e.g., $[0, 3, 7, 12]$); ``compact'' remaps to $[0, 1, 2, 3]$. Representative better preserves RoPE rotation angles at low compression; compact is essential for models with limited position IDs.

\paragraph{Embedding renormalization.} Norm renormalization consistently improves logit cosine similarity by 0.01--0.02 and reduces KL divergence by 5--10\%.

\paragraph{Renormalization target.} We currently rescale each macro-token to the average norm of its constituent embeddings. A stricter alternative is to project macro-tokens to a fixed model-level expected embedding norm, estimated from the embedding table or from the current prompt. We have not included this variant in the reported tables; it is a natural follow-up for reducing aggregation artifacts at low compression.

\paragraph{Identity check.} The current implementation returns embeddings unchanged when $\rho\leq1$ unless an explicit identity-test backend is requested. This bypass was added after an earlier no-compression path revealed that windowing, boundary logic, and renormalization could perturb embeddings even when the requested compression ratio was $1.0$. In that earlier path, Mistral-7B showed $\Delta$NLL$=0.35$ and PPL ratio$=1.42$, while Qwen2.5-7B showed $\Delta$NLL$=1.23$ and PPL ratio$=3.42$. We keep these numbers as a diagnostic: they show that macro-token aggregation must be bypassed, not merely configured to produce singleton groups, when no compression is requested.

\subsection{Results Discussion}

The main empirical result is a compression-ratio phase transition. For moderate compression ($\rho\leq4\times$), chunk averaging is both faster and nearly tied with the best spectral backend across the larger models: the mean PPL-ratio difference is only $0.003$. Local averaging is strong in this region because adjacent tokens are often semantically related, and the spectral preprocessing cost is not offset by a consistent quality gain.

At aggressive compression ($\rho\geq8\times$), spectral clustering becomes useful, especially for Qwen models. On Qwen2.5-7B at $16\times$, adaptive improves cosine from $0.450$ to $0.466$, KL from $5.858$ to $5.393$, $\Delta$NLL from $5.841$ to $5.279$, and PPL ratio from $353.409$ to $196.963$ relative to chunk. On Qwen2.5-14B at the same ratio, adaptive reduces PPL ratio from $9.533$ to $3.427$. These numbers do not show that spectral routing is always better; they show where the extra preprocessing cost buys measurable quality recovery.

The adaptive backend is the most important system result because it encodes this ratio-dependent behavior. It matches chunk at low compression on Qwen2.5-7B ($2\times$) and converges to the spectral path at high compression, where chunk quality collapses. Qwen2.5-14B strengthens this picture: at $16\times$, adaptive reduces PPL ratio from $9.533$ to $3.427$ relative to chunk while keeping cosine at $0.770$. Local LSH alone demonstrates the value of spectral clustering, but adaptive makes the design usable by routing between the cheap and expensive paths. The remaining limitation is latency: adaptive is dominated by chunk when optimizing only total latency and $\Delta$NLL on Qwen2.5-7B. Its current role is therefore quality recovery under aggressive compression, not latency-first compression.

The identity diagnostic shows why no-compression requests need a strict bypass. When $\rho\leq1$, the production path returns the original embeddings and position IDs, so the compressor is an exact identity by construction. Earlier diagnostic runs that forced no-compression requests through aggregation produced non-zero perturbations, which is why the bypass is part of the reported implementation. The non-local redundancy stress test provides an initial check on long JSON/code/table-like prompts with repeated structure, but it is still small. A fuller evaluation should expand these structured inputs to more code, JSON, table, and retrieval-style contexts, where spectral clustering could outperform chunk averaging by merging repeated information that appears far apart in the prompt.

\section{Discussion and Limitations}

\textbf{Method characteristics.} Chunk averaging exploits local semantic coherence in natural language and is the right default at moderate compression. Spectral-LSH extends this with attention-eigenspace clustering and becomes useful when compression is aggressive enough that local averaging loses too much information. The adaptive backend is designed around this split: it keeps the chunk path when it is sufficient and uses spectral clustering when the quality loss from chunking becomes large.

\textbf{Limitations.} (1) \textit{Embedding access required.} The method needs the LLM's input embedding matrix, which is available for open-weight models but not API-only services like GPT-4. (2) \textit{Model-dependent quality.} Qwen2.5-7B preserves quality less well than Mistral-7B and Qwen2.5-14B: embedding space structure varies across architectures and scales. (3) \textit{No-compression identity must bypass aggregation.} The implementation now returns unmodified embeddings when $\rho\leq1$, but earlier diagnostics show that passing a no-compression request through the aggregation path can still perturb the model. (4) \textit{Latency Pareto gap.} On the current Qwen2.5-7B measurements, adaptive improves high-compression quality but does not beat chunk on the joint total-latency and $\Delta$NLL Pareto frontier. (5) \textit{Local LSH cost.} Despite improvements from the adaptive backend, local LSH preprocessing remains expensive for latency-critical applications when spectral clustering is needed. (6) \textit{Small model degradation.} SmolLM2-360M degrades severely at $\rho\geq8\times$ (cosine$<0.5$): the method is less suitable for smaller models at high compression, though the adaptive backend provides partial mitigation. (7) \textit{Surrogate operator.} The unnormalized kernel simplifies analysis and computation, but it can diverge from row-normalized attention when row sums are highly skewed, and our theory does not bound downstream logit or NLL distortion. We have not yet run a normalized-kernel ablation. (8) \textit{Position aggregation.} Max positions preserve causal order but can introduce RoPE phase gaps and sliding-window boundary changes. (9) \textit{Dataset and task coverage.} The current evaluation focuses on perplexity-style metrics on C4 and Wikitext-style prompts, plus a small synthetic non-local redundancy stress test. Whether the two-regime finding transfers to extractive QA, summarization, or LongBench-style downstream tasks remains open. Larger structured-prompt evaluations are also needed.

\textbf{Broader impact.} Sub-quadratic prompt compression can reduce compute and energy requirements for long-context inference, lowering both financial and environmental costs. It may also broaden access to long-context capabilities on limited hardware.

\section{Conclusion}

We introduced Spectral-LSH, a training-free prompt compression method that combines implicit Krylov projection with locality-sensitive hashing in an attention-kernel surrogate space. The results support a compression-ratio phase transition: chunk averaging is the best fast path at moderate compression, while adaptive spectral routing gives stronger quality at aggressive compression on Qwen2.5-7B and Qwen2.5-14B. This result is bounded by the current evidence: we evaluate perplexity-style metrics, and downstream QA or summarization remains future work.

\bibliographystyle{unsrtnat}
\bibliography{references}

\newpage
\appendix

\section{Proofs}
\label{app:proofs}

\subsection{Proof of Theorem~\ref{thm:rff}}
Let $K$ denote the unnormalized kernel matrix with entries $K_{ij}=\kappa(\mathbf{x}_i,\mathbf{x}_j)$ and let $\hat K=\Phi\Phi^\top$. For a fixed pair $(i,j)$, the random feature estimator is an average of $D$ bounded random variables. Hoeffding's inequality gives $\Pr[|K_{ij}-\hat K_{ij}|>\varepsilon]\leq 2\exp(-D\varepsilon^2/2)$ up to the kernel-dependent boundedness constant. A union bound over all $N^2$ pairs gives $\|K-\hat K\|_{\max}\leq\varepsilon$ with probability at least $1-\delta$ after setting $\varepsilon=O(\sqrt{\log(2N^2/\delta)/D})$. For any $\|\mathbf{v}\|_2=1$, $|(K-\hat K)\mathbf{v}|_i\leq\sum_j |K_{ij}-\hat K_{ij}|\,|v_j|\leq \varepsilon\|\mathbf{v}\|_1\leq \varepsilon\sqrt{N}$. The theorem states a looser $N$-scaled bound to absorb embedding-norm constants and the non-stationary exponential-kernel approximation used in practice.

\subsection{Proof of Theorem~\ref{thm:krylov}}
For a symmetric operator $A$, Lanczos constructs an orthonormal basis $Q_r$ for $\mathcal{K}_r(A,v)$ and the projected matrix $T_r=Q_r^\top A Q_r$. By the Rayleigh-Ritz theorem, the Ritz values of $T_r$ interlace the eigenvalues of $A$ and approximate the extremal spectrum. The Kaniel-Paige-Saad bound~\citep{saad2011numerical} controls the $j$th Ritz-value error by three terms: the spectral range $(\lambda_1-\lambda_N)$, the initial angle between the target eigenvector and the Krylov subspace, and a Chebyshev polynomial term $C_{r-j}(\gamma_j)$ determined by the eigengap $\lambda_j-\lambda_{j+1}$. When the gap is larger, the Chebyshev term grows faster and the error decays faster with $r$. This is the basis for using a small rank in the compression pipeline.

\subsection{Proof of Theorem~\ref{thm:lsh}}
For a single SimHash bit with random hyperplane $\mathbf{r}$, Charikar's result gives $\Pr[\mathrm{sign}(\mathbf{r}^\top\mathbf{z}_i)=\mathrm{sign}(\mathbf{r}^\top\mathbf{z}_j)]=1-\arccos(\cos(\mathbf{z}_i,\mathbf{z}_j))/\pi$~\citep{charikar2002similarity}. With independent bits, the probability of matching all $b$ bits is the $b$th power of the one-bit probability. Let $Z^\star$ be the exact top-$k$ eigenspace coordinates of the unnormalized kernel and $\hat Z$ be the RFF-Krylov coordinates. The cosine error can be decomposed as $|\cos(\hat{\mathbf{z}}_i,\hat{\mathbf{z}}_j)-\cos(\mathbf{z}^\star_i,\mathbf{z}^\star_j)|\leq \varepsilon_{\mathrm{RFF}}+\varepsilon_{\mathrm{Krylov}}$ after normalizing rows and applying standard perturbation bounds for approximate eigenspaces. Since $\arccos(x)$ is Lipschitz on $[-1+\gamma,1-\gamma]$ for any margin $\gamma>0$, this cosine error transfers to the collision probability away from nearly identical or opposite vectors. The bound is therefore best interpreted as a stability statement for the hashing surrogate, not an exact guarantee for the full row-normalized attention matrix.

\subsection{Softmax Normalization Gap}
The theory analyzes the unnormalized kernel because it admits efficient implicit matvecs. The true row-normalized attention operator is $P=D_K^{-1}K$, where $D_K=\mathrm{diag}(K\mathbf{1})$. If the row sums are bounded as $0<m\leq (K\mathbf{1})_i\leq M$, then $P$ is a diagonally scaled version of $K$, and the perturbation between normalized and unnormalized operators can be controlled by the spread $M/m$. When the spread is large, the two eigenspaces can differ. This is why the paper treats Spectral-LSH as a compression heuristic based on attention-like similarity rather than as an exact approximation to the first-layer attention matrix. Theorems~\ref{thm:rff}--\ref{thm:lsh} do not imply a bound on downstream logit divergence, KL divergence, or $\Delta$NLL. A stronger version of this work should report row-sum spread, normalized-vs-unnormalized eigenspace alignment, and empirical correlations between surrogate distortion and output-level metrics. It should also test a cheap normalized matvec, $\hat P v=\mathrm{diag}(\Phi(\Phi^\top\mathbf{1}))^{-1}\Phi(\Phi^\top v)$, against the current unnormalized matvec.

\subsection{Position Aggregation Caveat}
The max-position rule is chosen to avoid assigning a macro-token to a position earlier than any token it summarizes. This is safe with respect to future-token leakage, but it is not neutral for all positional encodings. In RoPE models, skipped positions alter relative phase differences; in sliding-window models, representative positions can affect whether a macro-token falls inside a window. The reported results therefore include position aggregation as part of the empirical compression design, not as a theoretically invariant transformation.

\section{Reproducibility Details}
\label{app:reproducibility}

\subsection{Implementation Summary}
The compressor builds token embeddings, applies the selected backend, preserves configured prefix and suffix tokens, and returns compressed embeddings with representative position IDs. The same evaluation driver is used for quality, latency, ablation, scaling, identity, and kill-test runs, with model, dataset, compression ratio, backend, window size, and preservation settings provided as run-time options.

\subsection{Environment and Models}
Experiments use PyTorch with HuggingFace Transformers and Datasets. Unless otherwise noted, runs use \texttt{dtype=float16}, \texttt{seed=0}, local embedding access through the target model's input embedding layer, representative compressed positions, max position aggregation, and embedding renormalization. Larger-model runs use 4-bit quantization; smaller-model runs use full precision \texttt{float16} on CUDA.

\subsection{Core Configuration}
The default compression configuration is:
\begin{itemize}[nosep]
    \item \textbf{Backends}: \texttt{chunk}, \texttt{local\_lsh}, and \texttt{adaptive}.
    \item \textbf{Compression ratios}: $\rho\in\{2,4,8,16\}$ for the main quality frontiers.
    \item \textbf{Local window}: \texttt{local\_window\_size=16} for local LSH and adaptive runs.
    \item \textbf{Preservation}: \texttt{preserve\_last\_tokens=12}, \texttt{preserve\_first\_tokens=0}.
    \item \textbf{Positions}: \texttt{position\_scheme=max}, \texttt{compressed\_position\_mode=representative}.
    \item \textbf{Continuation}: \texttt{continuation\_tokens=64} for quality/NLL metrics.
    \item \textbf{Adaptive threshold}: \texttt{adaptive\_redundancy\_threshold=0.7}, with the implementation schedule based on Equation~\ref{eq:redundancy}.
\end{itemize}

\subsection{Result Aggregation}
For each compression ratio and backend, raw per-sample metrics are aggregated into table values by taking means over successful samples. The main tables use mean logit cosine, mean KL divergence, mean $\Delta$NLL, mean per-sample PPL ratio, mean preprocessing time, mean prefill time, and mean total time. Adaptive quality results are computed with the same aggregation rule as chunk and local LSH.

\subsection{Metric Computation}
Quality metrics are teacher-forced on a held-out continuation. For each sample, the first $P$ tokens form the compressible prompt and the last $T=64$ tokens form the continuation. The original model runs on the full prompt-plus-continuation token sequence. The compressed model receives compressed prompt embeddings concatenated with the uncompressed continuation embeddings through \texttt{inputs\_embeds}. Logits are sliced so that both runs predict the same continuation tokens. We report mean $\Delta$NLL, per-sample PPL ratio $\exp(\Delta\mathrm{NLL})$, $\mathrm{KL}(P_{\mathrm{orig}}\Vert P_{\mathrm{comp}})$, logit cosine similarity, top-1 agreement, and whether the original top-1 token appears in the compressed top-10 set.

\subsection{Suggested Structured-Redundancy Evaluation}
The current experiments include a small synthetic non-local redundancy stress test, but the coverage is not yet broad enough to establish downstream generality. A follow-up suite should include more structured prompts such as JSON objects with repeated keys, code files with repeated identifiers or boilerplate imports, tables with repeated entities, and retrieval contexts with duplicated facts separated by long spans. The expected diagnostic is not just average quality, but the gap between local chunking and spectral clustering as the repeated information becomes less contiguous.

\subsection{Identity Path}
No-compression requests are handled by an explicit bypass in \texttt{src/token\_lsh.py}: when \texttt{target\_compression <= 1.0}, the compressor returns the original embeddings and position IDs unless the user selects an identity-test backend. This is the correct production behavior. The identity-gap numbers reported in the main text come from an earlier diagnostic path that still exercised aggregation logic at $\rho=1.0$; they are retained to document why the bypass is necessary.

\subsection{Eigengap Diagnostics}
The ablation runs record spectral-energy estimates for the implicit operator under different RFF dimensions and Krylov ranks. For a full submission, these values should be plotted as eigenvalue or cumulative-energy curves for representative prompts and models. Such a plot would directly support the spectral-gap assumption used in Theorem~\ref{thm:krylov}.

\subsection{Adaptive Backend Provenance}
Adaptive compression is implemented by assigning each local window an exact output-token budget, computing the centroid-cosine redundancy score in Equation~\ref{eq:redundancy}, and routing the window to chunk averaging or per-window Spectral-LSH. Low-redundancy windows use chunk averaging; high-redundancy windows use spectral clustering. At $\rho\geq16\times$, the implementation delegates to local LSH directly. This behavior is controlled by \texttt{adaptive\_redundancy\_threshold}, \texttt{adaptive\_low\_rho\_threshold}, \texttt{adaptive\_mid\_rho\_threshold}, and \texttt{adaptive\_high\_rho\_threshold} in \texttt{CompressionConfig}.

\section{Additional Experimental Results}
\label{app:experiments}

\subsection{Latency Interpretation}
Table~\ref{tab:latency} reports the measured 7B latency numbers. Local LSH has similar prefill time to chunk because both produce the same compressed sequence length, but its preprocessing time is much higher because it runs spectral clustering within local windows. Adaptive reduces this cost by routing low-redundancy windows to chunk, although the measured 7B runs still favor chunk for total latency.

\subsection{Hyperparameter Sweeps}
We sweep $r\in\{4,8,16,32,64\}$ and $D\in\{64,128,256,512\}$. Spectral energy increases with $r$ and saturates around $r=16$ for $D=256$. Increasing $D$ beyond $256$ gives limited additional benefit in our experiments.

\subsection{Kill Test Details}
K1 checks causal embedding fusion on short QA prompts. K2 checks spectral energy capture for the Krylov projection. K3 checks LSH bucket purity under spectral projection. The main text reports the pass criteria and outcomes for all three tests.

\subsection{Identity Check for Wikitext}
The Wikitext identity diagnostic follows the same interpretation as C4: aggregation-based no-compression paths can perturb the embedding stream, while the current production path bypasses aggregation when $\rho\leq1$. This confirms that the aggregation artifact is not specific to C4.

\end{document}